\documentclass[11pt]{article}

\usepackage[letterpaper, left=1in, right=1in, top=1in,bottom=1in]{geometry}

\usepackage{parskip}

\usepackage{booktabs} 

\usepackage[utf8]{inputenc}

\usepackage{geometry}

\usepackage{graphpap,amscd,mathrsfs,graphicx,lscape,enumitem,dsfont,bm,url,subfigure}
\usepackage{epsfig,amstext,xspace}
\usepackage{algorithm,comment}
\usepackage{algorithmic}

\usepackage{auxiliary}
\usepackage{thmtools,thm-restate}

\usepackage{algorithm,algorithmic}
\usepackage[utf8]{inputenc} 
\usepackage[T1]{fontenc}    
\usepackage{hyperref}       
\usepackage{url}            
\usepackage{booktabs}       
\usepackage{amsfonts}       
\usepackage{nicefrac}       
\usepackage{microtype}      
\usepackage{graphicx} 
\usepackage{pgfplots}
\usepackage{tikz}
\usetikzlibrary{patterns}
\usepackage[font=footnotesize,labelfont=bf]{caption}

\usepackage{color}              
\usepackage[suppress]{color-edits}
\addauthor{tl}{cyan}
\addauthor{et}{blue}
\addauthor{dw}{brown}

\begin{document}
\title{Feedback graph regret bounds for Thompson Sampling and UCB}

 \author{
 Thodoris Lykouris\thanks{Microsoft Research NYC, \texttt{thlykour@microsoft.com}. Work mostly conducted while author was a Ph.D. student at Cornell University and supported by a Google Ph.D. Fellowship, and NSF grants CCF-1563714, and CCF-1408673.}
 \and  \'{E}va Tardos\thanks{Cornell University, \texttt{eva.tardos@cornell.edu}. Work supported in part by NSF grants CCF-1563714, CCF-1408673, and AFOSR grant F5684A1.}
 \and Drishti Wali \thanks{Cornell University, \texttt{drishtiwali@cs.cornell.edu}. Work supported in part by NSF grant CRII-1849899.}
 }
\date{}
\maketitle
\begin{abstract}
We study the stochastic multi-armed bandit problem with the graph-based feedback structure introduced by Mannor and Shamir \cite{MannorS11}. We analyze the performance of the two most prominent stochastic bandit algorithms, Thompson Sampling and Upper Confidence Bound (UCB), in the graph-based feedback setting. We show that these algorithms achieve  regret guarantees that combine the graph structure and the gaps between the means of the arm distributions. Surprisingly this holds despite the fact that these algorithms do not explicitly use the graph structure to select arms; they observe the additional feedback but do not explore based on it. Towards this result we introduce a \emph{layering technique} highlighting the commonalities in the two algorithms.
\end{abstract}


\section{Introduction}
\label{sec:intro}
Online learning is a classical model for sequential decision-making under uncertainty. At each time step the learner faces a choice between a set $\mathcal{V}$ of $k$ options usually referred to as arms. We consider the stochastic version of the  problem where there is a probability distribution $\mathcal{F}$ (fixed over time) of rewards over arms; we refer to the marginal distribution of arm $a$ as $\mathcal{F}(a)$. If the distribution $\mathcal{F}$ was known the decision-maker would always select the arm $a^\star$ with highest expected reward $\mu(a^\star)$. The goal of the learner is to make sequential choices while earning  rewards close to the rewards of arm $a^\star$. 

This trade-off between earning good rewards at the present (exploitation) and learning new information about the future (exploration) crucially relies on the information the learner receives as feedback. In the classical bandit model of online learning, the learner observes only the reward associated with her chosen action. This results in regret guarantees that scale with the number of arms. However in most applications of online learning the reward or loss of one arm reveals information about other arms which can significantly facilitate the learning process. A natural model capturing this extra information is the graph-based feedback setting of Mannor and Shamir~\cite{MannorS11} where the feedback is specified by a graph $G$ with the arms as its nodes. When an action $a$ is selected, the rewards of all arms adjacent to $a$ are revealed to the learner. In this setting, online learning techniques provide guarantees that scale with graph parameters for example, the independence number of graph $G$.

Classical stochastic bandit algorithms achieve enhanced performance guarantees when the difference between the mean of $a^\star$ and the means of other arms $a\in\mathcal{V}$ is large as then $a^\star$ is more easily identifiable as the best arm. This difference $\Delta(a)=\mu(a^\star)-\mu(a)$ is typically known as the \emph{gap} of arm $a$ and the performance guarantees scale inversely with it. There are two prominent practical stochastic bandit paradigms to derive these guarantees. The first is based on the idea of \emph{optimism in
the face of uncertainty} \cite{Lai+Robbins:1985,AuerCF2002,AudibertBubeck09,GarivierCappe11,Bubeck2013BanditsWH} which creates confidence intervals for the means of all arms and treats them as an optimistic estimate of their anticipated reward. Most of these algorithms are based on Upper Confidence Bound (UCB) algorithm of Auer et al.~\cite{AuerCF2002} which was also the first finite-time stochastic bandit algorithm. The second and more modern paradigm is based on randomized versions of these confidence intervals \cite{pmlr-v23-agrawal12,KauffmanMunosKorda12,RussoRoy14,Russo-Roy}. Thompson Sampling \cite{thompson_biom33} lies at the heart of most of this paradigm and has been proven useful in more complicated reinforcement learning settings \cite{AgrawalJia17}. However 
the only enhanced gap-based guarantees we have for these two important paradigms are for the pure bandit setting which does not incorporate richer notions of feedback such as the graph-based feedback\footnote{For other algorithms offering gap-based guarantees that incorporate the graph structure see related work.}. This poses the natural question:
\begin{quote}
    \emph{Can algorithms such as UCB and Thompson Sampling benefit from extra feedback?}
\end{quote}

\subsection{Our contribution}
We show that surprisingly these classical algorithms seamlessly combine the graph structure with the gaps of the arms to provide graph-based performance guarantees without any particular modifications. This is achieved despite the fact that they do not select arms specifically aiming to learn about the rewards of many other arms; they just incorporate the extra information that they happen to acquire via their selected neighbors. Our main result is to bound the regret of these algorithms in terms of $\sum_{a\in I} \frac{1}{\Delta(a)}$, where $I$ is an independent set of the graph $G$ and $\Delta(a)$ is the gap of arm $a$.

We assume that the feedback graph is fixed through time. The rewards of different time steps are independent but the rewards of different arms in any single time step may be correlated. Allowing such correlation makes the model more general since observations across possible actions are often strongly correlated: observations at nearby physical locations are likely similar, patients with similar profile may react to treatments in a similar way, effect of advertising is likely to be similar on similar observers, etc. We note that in many of these applications the feedback structure depends on physical structure of the alternatives and hence is not changing over time.  While revealing the reward about neighboring arms does not exactly model the information available to the learner in the above applications, the graph based feedback model is a simple and elegant abstraction of partial feedback and hence offers great opportunity to understand the effect of feedback structure on learning.

\paragraph{Our results.}
As a warm-up in Section~\ref{sec:active_arm_elimination} we show a regret guarantee of $\bigO\prn*{\max_{I\in\mathcal{I}}\sum_{a\in I}\frac{\log^2 T}{\Delta(a)}}$ where $\mathcal{I}$ is the set of all independent sets (Theorem~\ref{thm:active_fixed_graph_new}), for a graph-based variant of Active Arm Elimination \cite{EvenDarManMan06} similar to the one studied by Cohen et al.~\cite{DBLP:conf/icml/CohenHK16}. Although this result is weaker by a logarithm from the optimal bounds \cite{Buccapatnam-a} (see Section~\ref{ssec:related_work} for elaborate comparison to related work), its analysis serves as an important building block that allows us to extend the guarantees to UCB and Thompson Sampling. Our main results are then presented in Sections~\ref{sec:ucb} and \ref{sec:thompson_sampling} where we show how the aforementioned regret guarantees can be extended to UCB (Theorem~\ref{thm:ucb_main})  and Thompson Sampling (Theorem~\ref{thm:TS}) respectively.

\paragraph{Our techniques.} 
The warm-up algorithm in Section~\ref{sec:active_arm_elimination} selects arms that lie in a maximal independent set $I$ in a round-robin fashion. In one \emph{round} of this round-robin process we observe all the arms since at least one of their neighbors is in $I$ due to its maximality. This gives a gap-based upper bound on the number of times each suboptimal arm will be selected. For UCB and Thompson Sampling, we create a layering argument (Lemma~\ref{lem:layering}) that resembles these rounds. Unlike the rounds of Active Arm Elimination, the list of events in each \emph{layer} are not contiguous in time. When an arm $a^t$ is selected at time $t$, we place it in the lowest layer where it has not yet been observed, and place all its neighbors in the same layer (Figure~\ref{fig:layering}). The layers created this way have a few key properties that allow us to adapt the warm-up analysis of Active Arm Elimination
to this case:
\begin{itemize}
\item The arms put in a layer by being selected in the algorithm form an independent set.
\item At the time a selected arm is put in layer $\ell$, it has been observed at least $\ell-1$ times.
\end{itemize}
Thus, we can think of 
the layers as corresponding to rounds of the active arm elimination, and this enables us to extend the analysis to these algorithms.

\subsection{Related Work}\label{ssec:related_work}
The feedback graph structure for online learning was introduced in the adversarial setting \cite{MannorS11}. In this setting Alon et al. \cite{Mannor-Nonstochastic-MAB} show regret bounds of at most $O(\sqrt{T \beta\log k})$, where $\beta$ is the independence number of the graph. Subsequent work has focused on providing improved data-dependent guarantees \cite{Kock2014EfficientLB, LykourisSrTa18}, robustness to noise \cite{KocakNV16}, and understanding the effect of different observability structures \cite{AlonCBDK15,DBLP:conf/icml/CohenHK16}.

Stochastic multi-armed bandits as a model of online learning has a long history dating back to the seminal works of Lai and Robbins \cite{Robbins:1952, Lai+Robbins:1985}; in the finite-horizon setting, the first algorithm suggested was the Upper Confidence Bound (UCB) algorithm by Auer et al. \cite{AuerCF2002}. In the context of feedback graphs, stochastic bandits were first considered by Caron et al. \cite{UCB-clique} who provided the natural generalization of UCB, which they termed UCB-N where the neighbors of selected arms also make updates. The regret guarantee they obtain is of the form $\sum_{c\in C} \frac{(\max_{a\in c}\Delta(a))\cdot \log T}{(\min_{a\in c} \Delta(a))^2}$ where $C$ is the minimum-size partition of arms across cliques (clique cover). We improve upon this guarantee in multiple fronts. First, even though we lose an extra logarithm compared to this result, the maximum gap in any clique can be $1$ at every round, therefore our result has an improved dependence on the gaps (inverse linear instead of inverse quadratic). This in particular implies that our worst-case dependence on the time-horizon (ignoring logarithms) is $\sqrt{T}$ instead of $T^{2/3}$. Maybe even more importantly, our result sums over nodes in an independent set instead of a clique cover (the number of disjoint cliques needed to cover the graph). These quantities can be really far apart which gives an additional big improvement on gap-based bounds for UCB-N.

The first works going beyond clique partition as a parameter of the graph structure in the context of stochastic multi-arm bandits with feedback graphs are due to Buccapatnam et al. \cite{Buccapatnam:2014,Buccapatnam-a} and then Cohen et al. \cite{DBLP:conf/icml/CohenHK16}, both using variants of the Active Arm Elimination algorithm of Even-Dar et al. \cite{EvenDarManMan06}. Buccapatnam et al. \cite{Buccapatnam:2014,Buccapatnam-a} combine a version of eliminating arms 
suggested by Auer and Otner \cite{Auer2010} with linear programming to incorporate the graph structure in an algorithm they term UCB-LP\footnote{Despite the name, this algorithm is based on eliminating arms and does not select the arm with the higher upper confidence bound as the algorithm suggested by \cite{UCB-clique} which we study in Section~\ref{sec:ucb}.} which provides a regret guarantee of $\sum_{a\in D}\frac{\log T}{\Delta(a)}+k^2$ where $D$ is a particularly selected dominating set. Their algorithm uses the outcome of the linear program to explicitly guide exploration which is crucial in order to obtain a guarantee that depends on the minimum dominating set.\footnote{In Section~\ref{sec:conclusions}, we show that one cannot hope to obtain the same guarantee for algorithms such as UCB-N and TS-N that do not explicitly use the feedback graph to guide the exploration.} In contrast, our main contribution is to shed light on the ability of classical algorithms to seamlessly incorporate feedback without explicitly seeking to do so; in fact, we provide a unifying analysis for gap-based guarantees for algorithms such as UCB-N and TS-N that are more practical (for instance, they do not require knowledege of the time horizon, unlike techniques based on eliminating arms). Comparing the bounds, our approach depends on the possibly larger independence number (which is unavoidable for UCB-N and TS-N), loses an extra log factor, but is independent of $k$. Another work that utilizes the idea of eliminating arms for feedback graphs is the one by Cohen et al. \cite{DBLP:conf/icml/CohenHK16} who show a regret guarantee of $\sum_{a\in S} \frac{\log T}{\Delta(a)}$ for unknown and evolving graphs where $S$ is the set of the $\beta\log k$ arms with the smalles gap and $\beta$ is again the size of the maximum independent set. For the case of fixed graphs (e.g. capturing geographic proximity), we refine the above result to depend inversely on the gaps of a maximum independent set instead of the $\beta\log k$ smaller gaps. More importantly, our layering technique shows how such a result can be extended to more practical algorithms such as UCB-N and TS-N.

Thompson Sampling was initially suggested by Thompson \cite{thompson_biom33}; it was analyzed in the Bayesian setting (where we have priors for all arms) by Russo and Van Roy \cite{Russo-Roy} and in the frequentist setting (prior-free Bayesian setting) by Agrawal and Goyal \cite{pmlr-v23-agrawal12, Thompson-Sampling-Original,Shipra-2017}. In the context of undirected feedback graphs Tossou et al. \cite{Thompson-Sampling} and Liu et al. \cite{Info-sampling} extend the Bayesian guarantees incorporating the clique-cover size of the graphs in the natural graph extension of Thompson Sampling which they term TS-N. Recently Liu et al. \cite{Liu2018AnalysisOT} replace the latter with the independence number. The latter works also provide empirical comparisons of various stochastic bandit algorithms on different graphs and show the superiority of Thompson Sampling on the estimated graphs. However the regret bounds for all of \cite{Thompson-Sampling,Info-sampling,Liu2018AnalysisOT} incur a $\sqrt{T}$ dependence on the time horizon $T$. In contrast, we provide the first gap-dependent bounds for Thompson Sampling that go beyond the classical bandit setting and utilize the graph structure, while working on the more complicated frequentist setting. We note that a concurrent and independent work of Hu et al. \cite{HuMehtaPan19} also provides gap-dependent bounds for TS-N; their results are still weaker than ours since they scale with the clique cover rather than the independent set.

\section{Model}
\label{sec:model}
\paragraph{Multi-armed bandit with graph-based feedback.}
Our setting consists of a set $\mathcal{V}$ of $k$ arms and a probability distribution $\mathcal{F}$ of the rewards of the arms (where rewards of different arms may be correlated). Let $\mathcal{F}(a)$ be the marginal distribution of $\mathcal{F}$ for each arm $a\in\mathcal{V}$; we assume that this distribution has support only on $[0,1]$ and we denote its mean by $\mu(a)$. Crucially, the means of the different arms are unknown to the learner and the learner does not have prior distributional information about these means.

Whenever arm $a'$ is selected we sample an independent reward vector $r$ from the distribution $\mathcal{F}$, and earn reward $r(a')$. Let $a^\star$ denote the arm with the highest mean, and for each arm $a\in\mathcal{V}$ let $\Delta(a)=\mu(a^\star)-\mu(a)$ be the gap in expected rewards between the optimal arm $a^\star$ and the arm $a\in\mathcal{V}$.

The information feedback structure is defined by an undirected graph $G$  on the set of nodes $\mathcal{V}$. When the learner selects an arm $a'$, she receives reward $r(a')$, and also observes the rewards $r(a)$ for the set of arms $a\in\mathcal{N}(a')$, where $\mathcal{N}(a')$ denotes the set of nodes adjacent to $a'$ in the graph $G$. We use $\mathcal{I}(G)$ to denote the set of independent sets of $G$ and assume that the graph $G$ is fixed across time steps.

More formally, the protocol is as follows: We are given a set of arms $\mathcal{V}$, an undirected graph $G$ on these arms, and a time horizon $T$. The adversary selects the reward distribution $\mathcal{F}$ with rewards $r(a)\in [0,1]$ for all arms $a\in \mathcal{V}$. For each round $t=1,2,...,T$:
\begin{enumerate}
    \item  The learner selects an arm $a^t$ (possibly using a randomized algorithm).
    \item Stochastic rewards are drawn for all arms $a\in\mathcal{V}$: $r^t\sim\mathcal{F}$ (where rewards of different arms may be correlated).
    \item The learner earns reward $r^t(a^t)$, and observes the reward $r^t(a^t)$, as well as the rewards $r^t(a)$ for all arms $a\in \mathcal{N}(a^t)$, adjacent to $a^t$ in the graph $G$.
\end{enumerate}

\paragraph{Regret.}
The goal of the learner is to maximize the expected reward earned over time. If the distribution $\mathcal{F}$ was known, the learner would select $a^\star$ in every round, so we measure the performance of the learner by the expected regret, comparing its reward to the reward of the best arm
$$\bm{R_T}=\mathbb{E}
\brk*{\sum_{t}r^t(a^\star)
-r^t(a^t)}
,$$
where expectation is taken over the randomness of the rewards of the arms as well as the choices of the algorithm. For ease of presentation, we express the regret in terms of the gaps of the arms as
$$
\bm{R_T}=\sum_{t}\mathbb{E}
\brk*{\Delta(a^t)},
$$
where the expectation is now only over the choices of the algorithm.

\section{Warm-up: Active Arm Elimination via the layering technique}
\label{sec:active_arm_elimination}
In this section, we show how to adapt the Active Arm Elimination algorithm of Even-Dar et al.~\cite{EvenDarManMan06} using the graph structure to obtain regret bounds that only depend on the gaps of the nodes lying on an independent set.
The purpose of this section is to introduce our main technique, \emph{layering}, which serves as a building block for deriving the same guarantee for UCB (Section~\ref{sec:ucb}) and Thompson Sampling (Section~\ref{sec:thompson_sampling}) 
that do not explicitly use the graph structure. 

The Active Arm Elimination algorithm maintains the empirical mean $\tilde{\mu}^t(a)$ for each arm $a\in\mathcal{V}$ at each time step $t$ along with a confidence interval ensuring that the actual mean $\mu(a)$ falls within this interval with high probability at all times. An arm is eliminated if its confidence interval is fully below the interval of some other arm. The original Active Arm Elimination algorithm plays all not yet eliminated arms in a round robin fashion.

We adapt Active Arm Elimination by proceeding in rounds (the algorithm is formally described in Algorithm~\ref{alg:AAE}). In each round, we choose a maximal independent set of the not-yet eliminated arms\footnote{Maximal corresponds to an independent set that cannot be extended; such a set can be computed by adding nodes greedily. Note that an independent set in any subgraph is also independent in the original graph.} and we play once each node in this independent set, instead of all the non-eliminated arms as the original algorithm. By playing a maximal independent set in a round, we observe at least one sample for the reward of each arm, and hence improve the estimates of all arms. We note that any maximal independent set works well, so selecting an independent set greedily is fine.

We denote the set of active arms (that is, the set of non-eliminated arms)
$\mathcal{A}$ and use  $N_a^t$ to denote the number of times an arm $a$ has been observed until time step $t$. The empirical mean of an arm $a$ at the end of round $t$ is 
$$\tilde{\mu}^t(a)=\frac{1}{N_a^t}\sum_{\substack{s\le t: a^{s}=a \\ \textrm{ or }  a\in\mathcal{N}(a^s)}}
r^s(a)$$

As a confidence interval we use the interval centered around  $\tilde{\mu}^t(a)$ extended by $\sqrt{\ln(2Tk/\delta)/(2N_a^t)}$ in both directions. Using classical concentration bounds and the union bound we get that with high probability the mean of each arm falls within this interval (Lemma~\ref{lem:concentration_new}); for completeness we provide its proof in Appendix \ref{app:proof_concentration} of the supplementary material.

\begin{algorithm}[]
\caption{Active Arm Elimination using  independent set}
\label{alg:AAE}
\begin{algorithmic}
 \STATE Initialize the set of active arms as $\mathcal{A}=\mathcal{V}$, time as $t=1$, and rounds as $\gamma=0$.
 \WHILE{$t \leq T$} 
 \STATE Move to the next round: $\gamma\leftarrow \gamma+1$
 \STATE Select a maximal independent set $I_{\gamma}$ of the subgraph of set  $\mathcal{A}$
  \FOR{all $a\in I_{\gamma}$}
  \STATE Select arm $a^t=a$ and earn reward $r^t(a^t)$
 \STATE Observe the samples from all arms in $\mathcal{N}(a^t)$
  \STATE Move to the next time step: $t\leftarrow t+1$
 \ENDFOR
  \STATE Delete from the set of active arms $\mathcal{A}$ all arms $a'$ whose confidence interval is below the confidence interval of some other arm $a\in \mathcal{A}$:
  $$
  \tilde{\mu}^t(a') +\sqrt{\frac{\ln(2Tk/\delta)}{2N_{a'}^t}}< \max_{a\in \mathcal{\mathcal{A}}} \prn*{\tilde{\mu}^t(a)-\sqrt{\frac{\ln(2Tk/\delta)}{2N_{a}^t}}}
  $$
  \ENDWHILE
  \end{algorithmic}
\end{algorithm}
\vspace{0.1IN}
\begin{lemma}
\label{lem:concentration_new} For any arm $a$ and any time $t$, with probability at least $1-\frac{\delta}{kT}$ it holds that
$$|\tilde{\mu}^t(a)
-\mu(a)|\leq\sqrt{\frac{\ln(\frac{2Tk}{\delta})}{2N^t_a}}.$$  The probability this is true for all arms throughout the algorithm is at least $1-\delta$.
\end{lemma}

\paragraph{Layering technique.} The crux of our analysis lies in identifying and using two properties that the arms selected in one particular round, which we term \emph{layers}, satisfy. These properties are presented in the following definition and are crucial in extending the guarantees to UCB and Thompson Sampling (in the next two sections). 
\begin{definition}[Layering of selected arms]\label{defn:layering}
\vspace{0.1in}
All selected arms are placed in layers $\ell\in\crl*{1,2,\ldots}$. Arm $a^t$ is placed in the minimum layer $\ell$ such that it does not neighbor any arm already placed in layer $\ell$.
\end{definition}
For the active arm elimination algorithm we presented above (Algorithm~\ref{alg:AAE}), layers correspond to the respective rounds denoted by $\gamma$ there. We now note two important properties of the layers.
\begin{itemize}
\item Arms in the same layer must be independent of one another thereby forming an independent set. 
This is true as once an arm $a'$ is selected and put in a layer $\ell$, any neighbor $a\in\mathcal{N}(a')$ that is later selected, 
can no longer be placed in layer $\ell$ by definition of the layers.
\item When an arm $a$ is placed in layer $\ell$, it must have been observed at least $\ell-1$ times. This is true as $\ell$ is selected at the lowest layer in which the arm has not yet been observed. 
\end{itemize}
The key lemma of the layering technique is bounding the regret of all selected arms assuming that they are not selected after being observed too many times. In particular, let $\Lambda_a^t$ be the highest layer in which arm $a$ is placed until time step $t$ (upper bounding the times the arm is observed at any time it is selected). Then the following lemma gives a graph-based upper bound on the regret coming from all arms with appropriately bounded $\Lambda_a^t$:
\begin{lemma}\label{lem:layering}
\vspace{0.1in}
Let $L_a=\frac{L}{\Delta(a)^2}$ for all arms $a$ some value $L$. Let also $\Lambda_a^t$ be the highest layer arm $a$ is placed until time step $t$. Then  
\begin{align*}
    \sum_{t=1}^T\sum_{a\in\mathcal{V}}\mathbb{P}\brk*{\mathbf{1}\crl*{a^t=a, \Lambda_a^t\leq L_a}}\Delta(a)\leq 4\cdot \log(T) \cdot \max_{I\in\mathcal{I}(G)}\sum_{a\in I}\frac{L}{\Delta(a)}+1.
\end{align*}
\end{lemma}
\begin{proof}
For the purpose of our analysis, we group the layers into phases, where phase $\phi$ begins in the first layer $\ell$ such that no arm $a$ with $\Delta(a)> 2^{-\phi+1}$ is placed in any layer higher than $\ell$ and ends at the last layer $\ell'$ that still includes arms $a$ with $\Delta(a)> 2^{-\phi}$. All arms $a$ with gap $\Delta(a)\in (2^{-\phi},2^{-\phi+1}]$ are associated with phase $\phi$.

We now evaluate the contribution to the regret of the LHS from arms associated with phase $\phi$. All these arms have gap at most $2^{-\phi+1}$ which therefore upper bounds the expected regret at these steps. The LHS focuses on the event that these arms appear only in layers smaller than $\frac{L}{\Delta(a)^2}\leq \frac{L}{2^{-2\phi}}$. Letting $\mathcal{V}_{\phi}$ be the arms associated with phase $\phi$ and $G_{\phi}$ be the subgraph with only arms $\mathcal{V}_{\phi}$, the contribution from these arms in the LHS is:
\begin{align*}
 \sum_{t=1}^T\sum_{a\in\mathcal{V}_{\phi}}\mathbb{P}\brk*{\mathbf{1}\crl*{a^t=a, \Lambda_a^t\leq L_a}}\Delta(a)&\leq \max_{I\in\mathcal{I}(G_{\phi})}\sum_{a\in I}
\frac{L}{2^{-2\phi}}\cdot 2^{-\phi+1}
\leq 4 \cdot \max_{I\in\mathcal{I}(G)}\sum_{a\in I}\frac{L}{\Delta(a)}
\end{align*}
Phases $\phi\leq\log(T)$ each contribute one such term which leads to the additional $\log(T)$ in the RHS. For arms with $\Delta(a) \leq 1/T$, the expected regret using such arms is bounded by at most 1 overall.
\end{proof}
We now apply the previous lemma to directly show a regret guarantee based on the gaps of the independet sets for the active arm elimination algorithm. 
\begin{theorem}\label{thm:active_fixed_graph_new}
\vspace{0.1in}
Algorithm~\ref{alg:AAE} has expected regret bounded as
$$\bm{R_T}\leq 32
\cdot
\ln(2kT/\delta)\cdot\log(T)\underset{I\in
\mathcal{I}(G)}\max\sum_{a\in I}\frac{1}{\Delta(a)}+T\delta +1$$
Setting $\delta=\frac{1}{T}$, we obtain
a bound of $\bm{R_T}=\tilde{O}(\sum_{a\in I}\frac{1}{\Delta(a)})$ for some Independent Set $I$ of the underlying graph.
\end{theorem}
\begin{proof}
Recall that regret can be expressed as
$\bm{R_T}=\sum_{t}\mathbb{E}[\Delta(a^t)]$. It will be useful to write this as
$$\bm{R_T}=\sum_{t=1}^T\sum_{a\in\mathcal{V}}\mathbb{P}\brk*{\mathbf{1}\crl*{a^t=a}}\Delta(a)$$
To bound the regret, we first observe that by Lemma \ref{lem:concentration_new}, the probability that there exists an arm whose empirical mean fails to be in its corresponding confidence interval is bounded by $\delta$. The maximum regret we can get over $T$ steps is at most $T$ as rewards at each time step are bounded in $[0,1]$, so the unlikely event of an empirical mean falling outside the confidence interval (including also when the optimal arm is eliminated) contributes at most $\delta T$ to the expected regret. For the rest of the analysis we assume that the confidence intervals include the actual mean for each arm throughout the algorithm.

An arm $a$ is definitely eliminated when the upper bound of its confidence interval is below the lower bound of the confidence interval of $a^{\star}$. The distance between the actual mean and any of the lower or upper bounds of the confidence interval of an arm $a$ can differ by $2\cdot\sqrt{\frac{\ln(2Tk/\delta)}{2N^t_a}}$ as we assume that all means lie inside the confidence interval. Since the actual mean of arm $a$ and $a^{\star}$ differ by $\Delta(a)$, in order to ensure that arm $a$ is eliminated, the lower bound of $a^{\star}$ must be within $\frac{\Delta(a)}{2}$ of $\mu(a^{\star})$.
Similarly, the upper bound of $a$ must be within $\frac{\Delta(a)}{2}$ of $\mu(a)$.
To guarantee this we need that $\sqrt{\frac{\ln(2Tk/\delta)}{2N^t_a}}\le \frac{\Delta(a)}{4}$ and $\sqrt{\frac{\ln(2Tk/\delta)}{2N^t_{a^{\star}}}}\le \frac{\Delta(a)}{4}$. This happens when $N_a^t$ and $N_{a^\star}^t$ are both at least 
$$N_a^t, N_{a^\star}^t \ge \frac{8\ln(2Tk/\delta)}{\Delta(a)^2}.$$
Since, via layering, the arm is added to the smallest
 layer that it is not yet observed, the above implies that arm $a$ is never added to a layer larger than $\frac{L}{\Delta(a)^2}$ for $L=8\ln(2Tk/\delta)$. By Lemma~\ref{lem:layering}, when no confidence interval is violated, the regret is at most $4\cdot \log(T)\cdot\max_{I\in\mathcal{I}(G)}\sum_{a\in I}\frac{8\ln(2Tk/\delta)}{\Delta(a)}+1$. 
\end{proof}
We note that the round-robin version in the algorithm is, in fact, not necessary (see Remark~\ref{rem:aae_min_observations}). 
\begin{remark} \label{remark}
\vspace{0.1in}
In the above analysis, we discussed fixed graphs and provided regret guarantees based on independent set. In contrast, Buccapatnam et al. \cite{Buccapatnam:2014} use dominating set and Cohen et al. \cite{DBLP:conf/icml/CohenHK16} focus on evolving unknown graphs. Our bounds can extend in either of these directions by using a dominating set instead of an independent set in the algorithm and by sampling uniformly at random among active arms and applying Turan's theorem. However, using fixed graphs and independent set is crucial in extending our results beyond Active Arm Elimination (Thompson Sampling and UCB); this is why we present our analysis with respect to this setting. We note that one cannot hope for regret bounds based on the minimum dominating set for UCB and Thompson Sampling that do not use the feedback graph to explicitly target exploration as we discuss in Section~\ref{sec:conclusions}. \end{remark}

\section{Upper Confidence Bound }
\label{sec:ucb}
In this section, we present our first main result: combining gaps of the arms and the independent set of the graph $G$ for bounding the expected regret of UCB; in the next section we extend this to Thompson Sampling. Note that unlike our version of Active Arm Elimination in Section \ref{sec:active_arm_elimination} that explicitly selected independent sets neither UCB nor Thompson Sampling needs any change to adapt to the graph structure. 

The original UCB algorithm of Auer et al. \cite{AuerCF2002} is based on the same confidence intervals as Active Arm Elimination\footnote{To avoid using the time horizon $T$ in the algorithm, we can use the current time $t$ instead of $T$ in defining confidence intervals.}, but is using them in an optimistic way: at each iteration it selects the arm whose upper confidence bound is as high as possible. The natural extension of this with a graph feedback, suggested by Caron et al. \cite{UCB-clique} and termed UCB-N, selects the arm in precisely the same way but also updates the estimates of the neighbors of the selected arm. The algorithm is formally described in Algorithm~\ref{alg:UCB}.

\begin{algorithm}[]
\caption{UCB-N}
\label{alg:UCB}
\begin{algorithmic}
 \STATE Initialize time as $t=1$
 \WHILE{$t \le T$}
 \STATE $a^t=\underset{a\in\mathcal{V}}{\textrm{argmax}}\prn*{\Tilde{\mu}^t(a)+\sqrt{\frac{\ln (\frac{2kT }{\delta})}{2N_a^t}}}$
 \STATE Select arm $a^t$ and earn reward $r^t(a^t)$ 
 \STATE Observe the samples from all arms in $\mathcal{N}(a^t)$
   \STATE Move to the next time step: $t\leftarrow t+1$
  \ENDWHILE
  \end{algorithmic}
\end{algorithm}

We analyze the expected regret of the UCB-N algorithm by relating it to a run of the variant of Active Arm Elimination considered in Section \ref{sec:active_arm_elimination}. A round there corresponded to selecting arms of a maximal independent set over the arms not yet eliminated. We divide the run of UCB into layers where a layer corresponds to a round of Active Arm Elimination. When we select an arm, we place it in the minimum layer in which it has not yet been observed (see Definition~\ref{defn:layering}). We illustrate this layering construction pictorially in Figure \ref{fig:layering}, where the sequence of nodes as they are selected are put in layers 1, 2, and then layer 1 again despite being selected afterwards.
Although arms
in a layer are no longer selected contiguously,
Lemma~\ref{lem:layering}
shows that the layering technique still applies.  We formalize the regret guarantee 
in the following theorem. 
\begin{figure}[]
\centering
\begin{tikzpicture}
\tikzstyle{vertex}=[circle,fill=black!10,scale=0.8]
\node[vertex,fill=green](A) at (0,1){a};
\node[] (AR) at (0, 1.5) {\scriptsize $t_3$};
\node[vertex,fill=green](C) at (2,1){c};
\node[] (CR) at (2, 1.5) {\scriptsize $t_1$};
\node[vertex](B) at (1,1){b};
\node[] (BR) at (1, 1.5) {\scriptsize $t_1$};
\node[vertex](D) at (2.5,0){d};
\node[] (DR) at (2.5, 0.5) {\scriptsize $t_1$};
\node[vertex](E) at (1,0){e};
\node[] (ER) at (0.5, 0) {\scriptsize $t_1$};
\path[draw,thick,color=orange, -] (C) --  (E);
\path[draw,thick,-] (E) --  (A);
\path[draw,thick,-] (E) --  (B);
\path[draw,thick,color=orange,-] (C) --  (D);
\path[draw,thick,color=orange,-] (C) --  (B);
\node[] (I) at (1, 2) {\small Layer $1$};

\node[vertex](A1) at (6,1){a};
\node[] (AR1) at (6, 1.5) {\scriptsize $t_2$};
\node[vertex](C1) at (6+2,1){c};
\node[] (CR1) at (6+2, 1.5) {\scriptsize $t_2$};
\node[vertex](B1) at (6+1,1){b};
\node[] (BR1) at (6+1, 1.5) {\scriptsize $t_2$};
\node[vertex](D1) at (6+2.5,0){d};
\node[vertex,fill=green](E1) at (6+1,0){e};
\node[] (ER1) at (6+0.5, 0) {\scriptsize $t_2$};
\path[draw,thick,color=orange, -] (C1) --  (E1);
\path[draw,thick,color=orange,-] (E1) --  (A1);
\path[draw,thick,color=orange,-] (E1) --  (B1);
\path[draw,thick,-] (C1) --  (D1);
\path[draw,thick,-] (C1) --  (B1);
\node[] (I) at (6+1, 2) {\small Layer $2$};

\node[vertex](A2) at (12+0,1){a};
\node[vertex,fill=green](C2) at (12+2,1){c};
\node[] (CR2) at (12+2, 1.5) {\scriptsize $t_4$};
\node[vertex](B2) at (12+1,1){b};
\node[] (BR2) at (12+1, 1.5) {\scriptsize $t_4$};
\node[vertex](D2) at (12+2.5,0){d};
\node[] (DR2) at (12+2.5, 0.5) {\scriptsize $t_4$};
\node[vertex](E2) at (12+1,0){e};
\node[] (ER2) at (12+0.5, 0) {\scriptsize $t_4$};
\path[draw,thick,color=orange, -] (C2) --  (E2);
\path[draw,thick,-] (E2) --  (A2);
\path[draw,thick,-] (E2) --  (B2);
\path[draw,thick,color=orange,-] (C2) --  (D2);
\path[draw,thick,color=orange,-] (C2) --  (B2);
\node[] (I) at (12+1, 2) {\small Layer $3$};

\end{tikzpicture}
\caption{There are $k=5$ arms; $\crl{a,b,c,d,e}$. We show the first four steps $\crl{t_1,t_2,t_3,t_4}$ of the layering construction (for the first $3$ layers); the time next to a node denotes the first time it is observed in the layer. The nodes selected in these times are $c,e,a,c$; we denote these nodes by green. Orange edges show which nodes were observed for the first time in the layer. Note that, at time $t_3$, the selected node $a$ is put in the first layer despite having been observed in a higher layer (layer $2$). Also note that a node may be observed by multiple selected nodes in the same layer  (e.g. node $e$ in layer $1$);
 this does not interfere with our analysis as more observations 
 only help the concentration bounds.
 }
\label{fig:layering}
\end{figure}
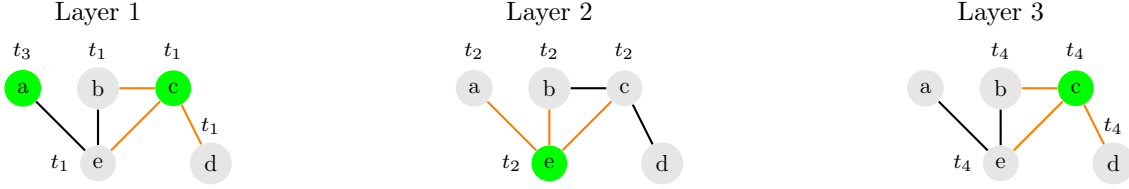

\vspace{0.3em}
\begin{theorem}
\label{thm:ucb_main}
The expected regret of the UCB-N algorithm (Algorithm~\ref{alg:UCB}) can be bounded as
$$\bm{R_T}\leq
8\cdot\ln(2kT/\delta)\cdot\log(T)\underset{I\in \mathcal{I}(G)}\max\sum_{a\in I}\frac{1}{\Delta(a)}+T\delta+1
$$
By setting $\delta=\frac{1}{T}$ we have $\bm{R_T}= \tilde{O}(\sum_{a\in I}\frac{1}{\Delta(a)})$ for an independent set $I$ of the graph.
\end{theorem}
\begin{proof}
As in the proof of Theorem \ref{thm:active_fixed_graph_new}, we start by pointing out that by Lemma \ref{lem:concentration_new} with probability at least $1-\delta$ the means of all the arms will stay in the confidence intervals around their empirical mean throughout the algorithm. The $\delta$ probability that this may fail can only contribute $\delta T$ to the expected regret, so for the rest of the analysis we will assume this does not happen. 

Recall that the Active Arm Elimination analysis was divided into phases, where in later phases arms with larger gaps are already eliminated. While UCB does not actively eliminate arms, we argue next that arms with large $\Delta$ values are not 
selected in high layers unless our assumption at the beginning of the proof about confidence intervals fails. By the definition of our confidence bounds and our assumption that the means of all arms remain in the confidence bounds throughout, once $$N_a^t \ge \frac{2\ln(2kT/\delta)}{\Delta(a)^2}$$
the upper confidence bound of arm $a$ is below the the mean of the optimal arm $a^\star$, and hence cannot be the arm selected by UCB. This comes from the same argument that was used in Active Arm Elimination, except we only need the upper bound for arm $a$ to stay below $\mu(a)+\Delta(a)$ and not $\mu(a)+\Delta(a)/2$ as was the case there. (This difference is what improves the bound by a factor of 4 compared to Theorem \ref{thm:active_fixed_graph_new}.)
In particular this implies that, when the confidence intervals are not violated, arm $a$ is never placed in any layer $\Lambda_a^t> \frac{L}{\Delta(a)^2}$ for $L=2\ln(2kT/\delta)$.
\footnote{If we use
the current time in defining confidence intervals, the confidence interval of an unseen arm will grow with time. This may cause the arm to be selected later; however, it will always go in a layer lower than the current bound. }

Similarly to the proof of Theorem~\ref{thm:active_fixed_graph_new}, applying Lemma \ref{lem:layering}, when no confidence interval is violated, the regret is at most $4\cdot \log(T)\cdot \max_{I\in\mathcal{I}(G)}\frac{2\ln(2Tk/\delta)}{\Delta(a)}+1$, which concludes the proof.
\end{proof}
By standard techniques for taking the worst case over $\Delta$'s, we also derive a 
gap-independent bound.
\vspace{0.1in}
\begin{corollary}
\label{cor:UCB}
The expected regret of UCB-N is bounded by  $2+ 4\sqrt{2\cdot\alpha T \ln(2kT^2)\cdot \log T}$ where $\alpha$ is the size of the maximum independent set. 
\end{corollary}
\begin{proof}
To get the gap-independent bound, we follow the standard bandit technique using Theorem~\ref{thm:ucb_main} for arms with gaps greater than some parameter $\Delta$.
\begin{align*}
    \bm{R_T}&=\sum_{t}\mathbb{E}[\Delta(a^t)]
    \leq\sum_{t:\Delta(a^t)>\Delta}\mathbb{E}[\Delta(a^t)] + T\Delta\\
    &\le 1+ T\delta + 8\cdot\ln(2kT/\delta)\log(T)\frac{\alpha}{\Delta} +T\Delta.
\end{align*}
which implies the result by choosing $\Delta=\sqrt{\frac{8\alpha\ln(2kT/\delta)\cdot\log(T)}{T}}$ and $\delta=1/T$.
\end{proof}
\begin{remark}\label{rem:aae_min_observations}
\vspace{0.1in}
In the previous section, we presented active arm elimination as selecting non-eliminated arms in a round-robin manner from an independent set. This presentation helps the exposition of the layering technique. However, we note that the above layering analysis can be used to show the same guarantee for a simpler
variant of Active Arm Elimination where we select the non-eliminated arm observed the fewest number of times, i.e., select the arm $\argmin_{a\in\mathcal{A}}N_a^t$. 
\end{remark}

\section{Thompson Sampling}
 \label{sec:thompson_sampling}
In this section,
we show that the Thompson Sampling algorithm of \cite{Thompson-Sampling-Original} also obtains similar guarantees. Similar to UCB, we do not alter the decisions of Thompson Sampling to accommodate the graph feedback structure but instead just update the information for neighbors of the selected arm. This natural extension, termed TS-N, was initially suggested in the Bayesian setting by Tossou et al.~\cite{Thompson-Sampling}. We now provide the main ingredients of this algorithm. 

\paragraph{TS-N algorithm.} The frequentist (prior-free) approach to Thompson Sampling starts with a Beta distribution {Beta}($\alpha,\beta$) for all arms with $\alpha=\beta=1$. A Beta distribution Beta$(\alpha,\beta)$ 
is defined with the following probability density function 
$$f_{\alpha,\beta}(x)=\frac{x^{\alpha-1}(1-x)^{\beta-1}}{B(\alpha,\beta)}$$ 
where $B(\alpha,\beta)$ is the normalization factor. At every time step $t$, the algorithm draws independent 
samples $\theta^t(a)$ from the 
Beta distribution of 
each arm $a$, selects the arm with the highest sample value and updates its posterior distribution using a Bernoulli trial with success probability equal to the reward obtained for this arm. The only change with graph feedback is that we also observe the reward for the neighbours of the selected arm, so we also update their distributions (see Algorithm \ref{alg:TS} for a formal description). The two key insights for using Beta distribution are that with the Bernoulli update used, its mean is the empirical mean of the rewards, and that the Bayesian posterior of a Bernoulli trial to a Beta distribution is also a Beta distribution.

\begin{algorithm}[]
\caption{Thompson Sampling with Graph Feedback}
\label{alg:TS}
\begin{algorithmic}
 \STATE Initialize the success and failure observed for each arm to zero; $S_a=0,F_a=0$ $\forall a\in\mathcal{V}$
 \STATE Initialize time $t=0$
 \WHILE{$t \le T$}
  \FOR{$a\in\mathcal{V}$}
  \STATE Sample $\theta^t(a)
  \sim$ Beta$(S_a+1,F_a+1)$
 \ENDFOR
 \STATE $a^t=\textrm{argmax}_{a\in\mathcal{V}}\theta^t(a)
 $;
 \STATE Select  arm $a^t$ and earn  reward $r^t(a^t)$
 \FOR{all arms $a=a^t$ or 
 $a \in \mathcal{N}(a^t)$}
 \STATE Perform a Bernoulli trial with success probability $r^t(a)$ and observe outcome $r_a^t\in\{0,1\}$
 \IF{$r^t_a=1$}
 \STATE $S_a=S_a+1$
    \ELSE
    \STATE $F_a=F_a+1$
    \ENDIF
 \ENDFOR
  \ENDWHILE
  \end{algorithmic}
\end{algorithm}

\paragraph{Outline of analysis of Thompson Sampling for Bandits.}
The general idea for analyzing stochastic bandits is to observe samples from all arms enough times to be confident that the empirical means are close enough to the actual means with high probability in order to identify the best arm. In Active Arm Elimination and UCB, we already showed that the regret incurred by the algorithm is only until all suboptimal arms have been observed enough times since thereafter, with high probability, only the optimal arm is selected. The regret in this case is generally $\sum_{a\in\mathcal{V}}\frac{\ln kT}{\Delta(a)}$ in a non-graph setting as observing any non-optimal arm $a$ at most $\frac{\ln kT}{\Delta(a)^2}$ times is sufficient for the empirical means to sufficiently concentrate.

Thompson Sampling is different in that the algorithm incurs regret
from two sources. Once the empirical means of the optimal arm $a^{\star}$ and the suboptimal arms are all concentrated well enough, the Thompson Sampling algorithm will also select the optimal arm with high probability. One source of regret is the usual regret incurred until all the suboptimal arms have been observed enough times. The other comes from the case where
the optimal arm has not been observed often enough; then its distribution is too diffuse which can cause a suboptimal arm to be selected. 

For the case of bandits, Agrawal and Goyal \cite{Shipra-2017} show that 
the expected number of times a suboptimal arm $a$ can be selected in this second case is bounded by $\frac{\ln kT}{\Delta(a)^2}$. Summing over all the arms they thus provide a regret incurred in this case by $\sum_{a\in\mathcal{V}}\frac{\ln kT}{\Delta(a)}$.

\paragraph{Our analysis.} We extend this analysis to obtain graph-based regret bounds similar to UCB. For the first case (in Lemma \ref{lm:TS-layering}) we use the layering argument of the previous subsection to bound the regret obtained from suboptimal arms $a$ until they have been observed at least $L_a:=\frac{16\ln kT}{\Delta(a)^2}$ times. We call a suboptimal arm $a$ \emph{saturated} if it has been observed at least $L_a$ times and \emph{unsaturated} otherwise. We define layers as we did for UCB: 
when we select an unsaturated arm $a$, we place the selected arm and its neighbors in the lowest layer the selected arm has not yet been 
observed. 
\vspace{0.3em}
\begin{lemma}\label{lm:TS-layering} The regret from selecting unsaturated arms is bounded by
$$\sum_{t=1}^T\sum_{a\in\mathcal{V}}\mathbb{P}\{a^t=a;N_a^t\leq L_a\}\Delta(a)\leq 64\cdot \log(kT)\log(T)\cdot\underset{I\in \mathcal{I}(G)}\max\sum_{a\in I}\frac{1}{\Delta(a)}+1$$
\end{lemma}
\begin{proof}
The proof follows by applying Lemma \ref{lem:layering} with $L=16\ln(kT)$, noticing that unsaturated arms are observed at most $L_a$ times and the maximum layer $\Lambda_a^t$ that an arm is ever placed is no greater than the number of times it is observed (see second property following Definition~\ref{defn:layering}).
\end{proof}

The part of the analysis more different for Thompson Sampling is bounding the regret incurred by selecting suboptimal arms $a$ after they are saturated. This can happen for one of two reasons:
\begin{itemize}
\item[(a.)] Despite having observed $a$ at least $L_a$ times ($N_a^t \ge L_a$), the sample $\theta^t(a)$ is significantly above the mean $\mu(a)$ of the arm $a$.
\item[(b.)] the sample $\theta^t(a^\star)$ is significantly below the mean $\mu(a
^\star)$ of the arm $a^\star$.
\end{itemize}
Similar to the analysis of Agrawal and Goyal \cite{Shipra-2017} we can show that option (a.) is unlikely, analogous to the unlikely events in UCB when the confidence intervals fail to contain the mean (see Lemma~\ref{lem:ts_conc_large_sample}). 
The additional novel part of the analysis is to avoid the dependence on the number of arms for case (b.). For that, we adapt the analysis in \cite{Shipra-2017} which bounds the expected number of times a suboptimal arm $a$ is selected by $\bigO \prn*{\frac{\ln kT}{\Delta(a)^2}}$. To prevent summing over all arms, we divide the arms into phases where a phase $\phi$ comprises of all arms with gaps in $[2^{-\phi
},2^{-\phi+1})$. This allows us to accumulate the regret from all arms in one phase $\phi
$ as $\frac{\ln T}{2^{-\phi
}}$ (Lemma~ \ref{lem:ts_conc_small_sample}). Summing across all possible phases provides a bound  depending only on the arm with the smallest gap $\Delta_{\min}$ instead of all the arms. The complete proof is provided in Theorem~\ref{thm:TS}.

We now address part (a.) by bounding the regret incurred from saturated suboptimal arms which were selected because their sample was significantly above their actual mean. 
 \vspace{0.3em}
\begin{lemma}\label{lem:ts_conc_large_sample}
The regret from selecting saturated arms $a$ with $\theta^t(a)
>\mu(a)+\frac{1}{2}\Delta(a)$ is bounded by 
$$\sum_{t=1}^T\sum_{a\in\mathcal{V}}\mathbb{P}\brk*{a^t=a,N_a^t\geq L_a,\theta^t(a)
>\mu(a)+\frac{1}{2}\Delta(a)} \Delta(a)\leq 2.$$ 
\end{lemma}
\begin{proof}
The proof is analogous to Lemma $7$  of Agrawal and Goyal  \cite{pmlr-v23-agrawal12}. Let $\tilde{\mu}^t(a)$ be the empirical mean of arm $a$ till time $t$. For an arm $a$, $\theta^t(a)>\mu(a)+\frac{\Delta(a)}{2}$, can only happen due to two reasons: 
\begin{itemize}
\item[(i)] $\tilde{\mu}^t(a)>\mu(a)+\frac{\Delta(a)}{4}$, 
\item[(ii)] $\theta^t(a)
>\tilde{\mu}^t(a)+\frac{\Delta(a)}{4}$
\end{itemize}
Both are unlikely if the arm $a$ has been observed at least $L_a$ times; the first by a Chernoff bound and the second by properties of the Beta distribution. We formalize these arguments in Appendix~\ref{app:proof_ts_large_sample}.
\end{proof}
Next we bound the regret due to part (b.): regret incurred by selecting a  saturated suboptimal arm $a$ due to the fact that the optimal arm has a sample significantly below its actual mean. We adapt the analysis from \cite{pmlr-v23-agrawal12}.

\vspace{0.3em}
\begin{lemma}\label{lem:ts_conc_small_sample}
Let $\mathcal{V}_{\phi}$ denote subset of arms
$\mathcal{V}_{\phi}=\{a\in \mathcal{V}: 2^{-\phi} \le \Delta(a) <2^{-\phi+1}\}$ for $\phi
>0$. The loss of these arms $a\in\mathcal{V}_{\phi}$ after being saturated but having sample $\theta^t(a)$ not too far from their actual means is bounded by
$$\sum_{t=1}^T\sum_{a\in\mathcal{V}_{\phi
}}\mathbb{P}\brk*{a^t=a, N^t_{a}\geq L_{a}, \theta^t(a)\leq \mu(a)+\frac{\Delta(a)}{2}}\Delta(a)\leq O(\frac{\ln T}{2^{-\phi}}).$$ 
\end{lemma}
\begin{proof}
To bound this term, we use the fact that the samples of the optimal arm between two consecutive observations of it come from the same Beta distribution, since the distribution is not updated in between. We use the technique from Agrawal and Goyal \cite{Thompson-Sampling-Original} to bound the probability that the optimal arm has its sample far below its actual mean. This allows us to bound the number of times an arm $a\in\mathcal{V_{\phi}}$ can be selected while its sample is close to its mean because the sample of the optimal arm $a^\star$ is far enough below its mean $\mu(a^\star)$. We formalize the arguments in Appendix~\ref{app:proof_ts_small_sample}.
\end{proof}
\begin{theorem}
\label{thm:TS}
The expected regret of the TS-N algorithm (Algorithm \ref{alg:TS}) is bounded by
\begin{align*} 
    \bm{R_T}\leq \bigO\Bigg(\underset{I\in \mathcal{I}(G)}\max \sum_{a\in I}\frac{\ln(T)\ln(kT)}{\Delta(a)}\Bigg)
\end{align*}
\end{theorem}
\begin{proof}
We bound the regret incurred by the algorithm in two parts: regret of arms $a$, while they are not saturated $N_a\le L_a$, and the regret of arms played after being saturated. The first part is bounded by Lemma \ref{lm:TS-layering}, while Lemmas~ \ref{lem:ts_conc_large_sample} and \ref{lem:ts_conc_small_sample} are used to bound the second part. More formally, we write the expected regret as
\begin{align*}
   \bm{R_T}&=\sum_{t=1}^T\mathbb{E}[\Delta(a^t)]
    =\sum_{t=1}^T\brk*{\sum_{a\in\mathcal{V}}\mathbb{P}\prn*{a^t=a;N_a^t\leq L_a}}\Delta(a)
    +\sum_{t=1}^T\brk*{\sum_{a\in\mathcal{V}}\mathbb{P}(a^t=a;N_a^t\geq L_a)}\Delta(a)
    \end{align*}
The first term is bounded by $64\max_{I \in \mathcal{I}(G)}\sum_{a \in I} \frac{\ln (T) \ln (kT)}{\Delta(a)}+1$ by Lemma \ref{lm:TS-layering}. To bound the second term we use we split this regret into two parts, separating the part when the sample of arm $a^t$ is far from its actual mean, and when it is not.
\begin{align*}
    \sum_{t=1}^T\sum_{a\in\mathcal{V}}\mathbb{P}\brk*{a^t=a;N_a^t\geq L_a}\Delta(a)
    &=\sum_{t=1}^T\sum_{a\in\mathcal{V}}\mathbb{P}\brk*{a^t=a,N_{a}^t\geq L_{a},\theta^t(a)> \mu(a)+\frac{\Delta(a)}{2}}\Delta(a)\\
    &+\sum_{t=1}^T\sum_{a\in\mathcal{V}}\mathbb{P}\brk*{a^t=a,N_{a}^t\geq L_{a},\theta^t(a)\leq \mu(a)+\frac{\Delta(a)}{2}}\Delta(a)
    \end{align*}
By Lemma \ref{lem:ts_conc_large_sample} the first part is bounded by 2. The second part can be rewritten as $$\sum_{t=1}^T\sum_{\phi=0}^{-\log \Delta_{\min}}\sum_{a\in\mathcal{V_{\phi}}}\mathbb{P}\brk*{a^t\in\mathcal{V}_{\phi
},N_{a}^t\geq L_{a},\theta^t(a)\leq \mu(a)+\frac{\Delta(a}{2}}\Delta(a)$$
where $\Delta_{\min}$ to denote the smallest gap on a non-optimal arm.
By Lemma \ref{lem:ts_conc_small_sample} this is bounded by 
\begin{align*}
& \sum_{\phi
=1}^{-\log_2 \Delta_{\min}}\bigO\prn*{\frac{\ln T}{2^{-\phi}}}
=\bigO\prn*{\frac{\ln T}{ \Delta_{\min}}}
\end{align*}
Combining the above bounds we obtain: 
\begin{align*}
    \bm{R_T}&=\bigO\prn*{\frac{\ln(T)}{\Delta_{\min}}+2+ \underset{I\in \mathcal{I}(G)}\max\sum_{a\in I}\frac{\ln(T)\ln(kT)}{\Delta(a)}}\\
    &=\bigO\prn*{\underset{I\in \mathcal{I}(G)}\max\sum_{a\in I}\frac{\ln(T)\ln(kT)}{\Delta(a)}}
\end{align*}
\end{proof}
As was done for Corollary \ref{cor:UCB} we can derive a gap independent bound.
\vspace{0.3em}
\begin{corollary}
The expected regret of the Thompson Sampling algorithm can be bounded as $\bigO(\sqrt{\alpha T \ln T \ln(kT)})$ where $\alpha$ is the size of the maximum independent set. 
\end{corollary}

\section{Conclusion}
\label{sec:conclusions}
In this paper, we analyze the performance of Thompson Sampling and UCB in the graph-based feedback setting. We bound the regret using the gaps of arms in an independent set, despite the fact that these algorithms do not explicitly use the graph structure to select arms. Below we discuss the results and suggest avenues for future research.
\begin{itemize}
\item In contrast to our results, Buccapatnam et al.~\cite{Buccapatnam:2014} offer an algorithm with regret bounded by the smallest dominating set of the graph and provide a lower bound based on fractional dominating set. It is not hard to see that the regret of both UCB-N and TS-N scales with the maximum independent set, and not the minimum dominating set of the graph. Consider a star graph with one optimal external node, and all others arms having similar gaps.  When running TS-N initially all arms use the same Beta distribution, but over time the central arm is observed most, it concentrates fast and once its distribution is concentrated, TS-N will select one of the spokes, each of which is sampling a more diffuse distribution. This reduces the algorithm to the bandit setting. A deterministic version of this argument applies for UCB-N.
\item On the negative side, our results suffer an extra logarithm compared to the results of Buccapatnam et al. \cite{Buccapatnam:2014}. This extra logarithm seems necessary if one approaches the problem via an argument based on phases (Cohen et al.\cite{DBLP:conf/icml/CohenHK16} also suffer from it due to the same reason). Understanding whether the extra logarithm is inherent to the algorithms of TS-N and UCB-N or is a shortcoming of our analysis is an interesting open question.
\end{itemize}

\bibliographystyle{alpha}
\bibliography{bib1}

\appendix

\section{Supplementary material from Section \ref{sec:active_arm_elimination}}\label{app:proof_concentration}

\textbf{Lemma~\ref{lem:concentration_new} restated.}
For an arm $a$ and any time $t$ 
$$\abs{\tilde{\mu}^t(a)
-\mu(a)}\leq\sqrt{\frac{\ln(\frac{2Tk}{\delta})}{2N^t_a}}$$ with probability at least $1-\frac{\delta}{kT}$, and the probability this is  true for all arms throughout the algorithm is at least $1-\delta$.

\begin{proof} The claim is that for each arm and for every time step, the actual  mean is within the confidence interval of its empirical mean. This comes from applying Hoeffding's concentration inequality for each arm and then from taking union bound over all arms and all time steps with high probability all arms remain in their confidence intervals.

To apply Hoeffding's inequality, consider the empirical mean as the sum of independent samples from the marginal distribution $\mathcal{F}(a)$. By Hoeffding's inequality, it holds that $$\textrm{Pr}\Bigg[\Bigg|\Big(\frac{1}{N_a^t}\sum_{\substack{s\le t: a^{s}=a \\ \textrm{ or }  a\in\mathcal{N}(a^s)}}r^s(a)\Big)-\mu(a)\Bigg|>c\Bigg]\leq 2e^{-2N^t_ac^2}$$
To bound the failure probability by $\frac{\delta}{kT}=2e^{-2N^t_ac^2}$, we set $c=\sqrt{\frac{ln(\frac{2kT}{\delta})}{2N^t_a}}$. Then, $$\textrm{Pr}\Bigg[|\tilde{\mu}^t(a)-\mu(a)|\leq\sqrt{\frac{ln(\frac{2kT}{\delta})}{2N^t_a}}\Bigg]\geq 1-\frac{\delta}{kT}$$ The proof then follows by applying union bound across all arms and time steps.
\end{proof} 
\section{Supplementary material from Section
\ref{sec:thompson_sampling}}
In this section, we  provide the proofs of Lemmas \ref{lem:ts_conc_large_sample} and \ref{lem:ts_conc_small_sample}.

\subsection{Proof of Lemma~\ref{lem:ts_conc_large_sample}}
\label{app:proof_ts_large_sample}
\textbf{Lemma~\ref{lem:ts_conc_large_sample} restated.}
The regret from selecting saturated arms $a$ with $\theta^t(a)
>\mu(a)+\frac{1}{2}\Delta(a)$ is bounded by
$$\sum_{t=1}^T\sum_{a\in\mathcal{V}}\mathbb{P}\brk*{a^t=a,N_a^t\geq L_a,\theta^t(a)
>\mu(a)+\frac{1}{2}\Delta(a)} \Delta(a)\leq 2.$$

\begin{proof}
The proof is analogous to Lemma $7$  of Agrawal and Goyal  \cite{pmlr-v23-agrawal12}. Let $\tilde{\mu}^t(a)$ be the empirical mean of arm $a$ till time $t$. For an
arm $a$,
$\theta^t(a)
>\mu(a)+\frac{\Delta(a)}{2}$, can only happen due to two reasons: 
\begin{itemize}
\item[(i)] $\tilde{\mu}^t(a)>\mu(a)+\frac{\Delta(a)}{4}$, 
\item[(ii)] $\theta^t(a)
>\tilde{\mu}^t(a)+\frac{\Delta(a)}{4}$
\end{itemize}
Both are unlikely if the arm $a$ has been observed at least $L_a$ times; the first by a Chernoff bound and the second by properties of the Beta distribution. More formally,
\begin{align*} 
  & \sum_{t=1}^T \sum_{a\in\mathcal{V}} \mathbb{P}\brk*{a^t=a,N_a^t\geq L_a,\theta^t(a)
  > \mu(a)+\frac{\Delta(a)}{2}
  }\Delta(a)\\
 &\le \sum_{t=1}^T\sum_{a\in\mathcal{V}}  \mathbb{P}\brk*{N_a^t\geq L_a, \tilde{\mu}^t(a)>\mu(a)+\frac{\Delta(a)}{4}}\Delta(a)\\
 &+\sum_{t=1}^T \sum_{a\in\mathcal{V}} \mathbb{P}\brk*{N_a^t\geq L_a,\theta^t(a)
 >
 \tilde{
 \mu}^t(a)+\frac{\Delta(a)}{4}}\Delta(a)
 \end{align*}
 Now, by Hoeffding's inequality, for any arm $a$ and time $t$
 $$
 \mathbb{P}\brk*{N_a^t\geq L_a, \tilde{
 \mu}^t(a)>\mu(a)+\frac{\Delta(a)}{4}}
 \leq e^{\frac{-2N_a^t\Delta(a)^2}{16}} \leq  e^{\frac{-2L_a\Delta(a)^2}{16}} 
 $$
 Now using the fact that $\Delta(a)\leq 1$ and the definition of $L_a=16\frac{\ln(kT)}{\Delta(a)^2}$, the first term inside the summation can be bounded as \begin{align*}
     \sum_{t=1}^T\sum_{a\in\mathcal{V}} \mathbb{P}\brk*{N_a^t\geq L_a, \tilde{
 \mu}^t(a)>\mu(a)+\frac{\Delta(a)}{4}}\Delta(a)
 &\leq T \sum_{a\in\mathcal{V}}  e^{\frac{-2L_a\Delta(a)^2}{16}}\leq Tk \cdot \frac{1}{k^2T^2} = \frac{1}{kT}.
 \end{align*}
 To bound the second term inside the summation for each arm $a$ and time $t$, we look at the sample from the underlying beta distribution at any time step $t$.  Let $S_a^t$ and $F_a^t$ be the successes and failures of Beta distribution at time step $t$. 
 \begin{align*}
 \mathbb{P}\brk*{N_a^t\geq L_a, \theta^t(a)>\tilde{\mu}^t(a)+\frac{\Delta(a)}{4}}
 &=\sum_{\ell=L_a}^T\mathbb{P}\brk*{N_a^t=\ell, \theta^t(a)>\tilde{\mu}^t(a)+\frac{\Delta(a)}{4}}\\
 &=\sum_{\ell=L_a}^T\mathbb{P}\brk*{S_a^t+F_a^t=\ell}\cdot\mathbb{P}\brk*{ \theta^t(a)>\tilde{\mu}^t(a)+\frac{\Delta(a)}{4}|S_a^t+F_a^t=\ell}\\
 &=\sum_{\ell=L_a}^T\mathbb{E}_{S_a^t+F_a^t=\ell}\brk*{1-F_{S_a^t,F_a^t}^{Beta}\Bigg(\tilde{\mu}^t(a)+\frac{\Delta(a)}{4}\Bigg)}\\
 \end{align*}
 where  $F_{S,F}^{Beta}(y)$ is the cumulative density function of the Beta distribution with probability density function $f_{S,F}$ as defined in Section~\ref{sec:thompson_sampling}. Now, we use a useful fact  about the Beta distributions (Fact $1$ from Agrawal and Goyal \cite{pmlr-v23-agrawal12}): 
 $$F_{S+1,F+1}^{Beta}(y)=1-F_{S+F+1,y}^{Binom}(S)$$ Here $F_{n,p}^{Binom}(\cdot)$ is the cumulative density function of the Binomial distribution with $n$ trials and trial success probability $p$. Thus, combining the above with the fact that the number of successes is equal to the number of observations times the empirical mean, $S_a^t=N_a^t\cdot \tilde{\mu}^t(a)$, we obtain: 
 \begin{align*}
 \mathbb{P}\brk*{N_a^t\geq L_a, \tilde{\mu}^t(a)>\mu(a)+\frac{\Delta(a)}{4}}
 &=\sum_{\ell=L_a}^T\mathbb{E}_{S_a^t+F_a^t=\ell}\brk*{F_{\ell+1,\tilde{\mu}(a)+\frac{\Delta(a)}{4}}^{Binom}\Bigg(\ell\tilde{\mu}(a)\Bigg)}\\
 &\leq \sum_{\ell=L_a}^T\mathbb{E}_{S_a+F_a=\ell}\brk*{F_{\ell,\tilde{\mu}(a)+\frac{\Delta(a)}{4}}^{Binom}\Bigg(\ell\tilde{\mu}(a)\Bigg)}\\
 &\leq \sum_{\ell=L_a}^T\mathbb{E}_{S_a+F_a=\ell}\brk*{e^{\frac{-2\Delta(a)^2\ell}{16}}}\leq \frac{1}{k^2T}.
 \end{align*} 
 The last inequality comes from Hoeffding inequality and the second-to-last inequality holds by an observation about Binomial distribution c.d.f. by Agrawal and Goyal (proof of Lemma 5 in \cite{pmlr-v23-agrawal12}): $$F_{n+1,p}^{Binom}(r)= (1-p) F_{n,p}^{Binom}(r)+pF_{n,p}^{Binom}(r-1)\leq  (1-p)F_{n,p}^{Binom}(r)+pF_{n,p}^{Binom}(r)\leq  F_{n,p}^{Binom}(r).$$ 
 Summing over all time steps and all arms, combining the bounds for both summands, and using that $k\geq 1$, completes the proof.
\end{proof}

\subsection{Proof of Lemma~\ref{lem:ts_conc_small_sample}}
\label{app:proof_ts_small_sample}
\textbf{Lemma~\ref{lem:ts_conc_small_sample} restated.}
Let $\mathcal{V}_{\phi}$ denote subset of arms
$\mathcal{V}_{\phi}=\{a\in \mathcal{V}: 2^{-\phi} \le \Delta(a) <2^{-\phi
+1}\}$ for $\phi>0$. The loss of these arms $a\in\mathcal{V}_{\phi}$ after being saturated but having sample $\theta^t(a)$ not too far from their actual means is bounded by
$$\sum_{t=1}^T\sum_{a\in\mathcal{V}_{\phi}}\mathbb{P}\brk*{a^t=a, N^t_{a}\geq L_{a}, \theta^t(a)\leq \mu(a)+\frac{\Delta(a)}{2}}\Delta(a)\leq \bigO\prn*{\frac{\ln T}{2^{-\phi}}}.$$ 
Before proving the lemma, we provide two useful lemmas that will help in the proof.

\begin{lemma}\label{lem:ts_suboptimal_to_optimal}
Let $\mathcal{V}_{\phi}$ denote subset of arms
$\mathcal{V}_{\phi}=\{a\in \mathcal{V}: 2^{-\phi
} \le \Delta(a) <2^{-\phi
+1}\}$ for $\phi>0$ and $\mathcal{H}_{t-1}$ be the history of the algorithm until time step $t-1$.
The probability of these arms $a\in\mathcal{V}_{\phi}$ being selected after being saturated while having sample $\theta^t(a)$ not too far from their actual means is bounded by 
$$\mathbb{P}\brk*{a^t\in\mathcal{V}_{\phi
},N^t_{a^t}\geq L_{a^t},\theta^t(a^t)\leq \mu(a^t)+\frac{\Delta(a^t)}{2}|\mathcal{H}_{t-1}}\leq \Bigg(\frac{1}{p_{\phi,t}}-1\Bigg)\mathbb{P}\brk*{a^t=a^\star|\mathcal{H}_{t-1}}$$
where $p_{\phi,t}=\mathbb{P}\brk*{\theta^t(a^\star)>y_{\phi}|\mathcal{H}_{t-1}}$ and $y_{\phi}=\max_{a\in\mathcal{V}_{\phi}}\Bigg(\mu(a)+\frac{\Delta(a)}{2}\Bigg)$
\end{lemma}
\begin{proof}
We bound the two sides of the inequality separately. 
\begin{align*}
\mathbb{P}\brk*{a^t\in\mathcal{V}_{\phi},N^t_{a^t}\geq L_{a^t},\theta^t(a^t)\leq \mu(a^t)+\frac{\Delta(a^t)}{2}|\mathcal{H}_{t-1}}
 \leq \mathbb{P}\brk*{a^t\in\mathcal{V}_{\phi
 },\theta^t(a^t)\leq \mu(a^t)+\frac{\Delta(a^t)}{2}|\mathcal{H}_{t-1}}.
 \end{align*}
 Since $a^t$ is the selected arm and thus has the highest valued sample $\theta^t(a^t)$, the samples of all other arms must be less than its sample and thus also less than $\mu(a^t)+\frac{\Delta(a^t)}{2}$ and the above is less than
 \begin{align*}
 \mathbb{P}\brk*{a^t\in\mathcal{V}_{\phi}, \theta^t(a)\leq \mu(a^t)+\frac{\Delta(a^t)}{2}: \forall a\in\mathcal{V} |\mathcal{H}_{t-1}}
 \leq \mathbb{P}\brk*{\theta^t(a)\leq y_{\phi}: \forall a\in\mathcal{V}|\mathcal{H}_{t-1}}.
 \end{align*}
 Now since we are conditioning on the history $\mathcal{H}_{t-1}$, the samples across arms are independent and therefore this is equal to:
 \begin{align*}
 \mathbb{P}\brk*{\theta^t(a^\star)\leq y_{\phi}|\mathcal{H}_{t-1}}\cdot \mathbb{P}\brk*{\theta^t(a)\leq y_{\phi }: \forall a\neq a^\star|\mathcal{H}_{t-1}}
 &=(1-p_{\phi ,t})\cdot \mathbb{P}\brk*{\theta^t(a)\leq y_{\phi }: \forall a\neq a^\star|\mathcal{H}_{t-1}}.
\end{align*}
We are now left to show that $$ \mathbb{P}\brk*{ \theta^t(a)\leq y_{\phi
}: \forall a\neq a^\star | \mathcal{H}_{t-1}}\leq\frac{1}{p_{\phi
,t}}\cdot  \mathbb{P}\brk*{a^t=a^\star|\mathcal{H}_{t-1}},$$
which holds because
\begin{align*}
    \mathbb{P}\brk*{a^t=a^\star|\mathcal{H}_{t-1}}
   &\geq \mathbb{P}\brk*{\theta^t(a^\star)>y_{\phi
   }\geq \theta^t(a): \forall a\neq a^\star | \mathcal{H}_{t-1}}\\
   &=\mathbb{P}\brk*{\theta^t(a^\star)>y_{\phi
   } | \mathcal{H}_{t-1}}\cdot \mathbb{P}\brk*{ \theta^t(a)\leq y_{\phi
   }: \forall a\neq a^\star | \mathcal{H}_{t-1}}\\
   &=p_{\phi
   ,t}\cdot \mathbb{P}\brk*{ \theta^t(a)\leq y_{\phi
   }: \forall a\neq a^\star | \mathcal{H}_{t-1}}.
\end{align*}
The first equality holds because the probabilities are conditioned on the history $\mathcal{H}_{t-1}$ and hence the samples of all arms are independent of one another.
\end{proof}

\begin{lemma}[Lemma $2.9$ in \cite{Shipra-2017}]\label{lem:expected_observations}
Let $\mathcal{H}_{t-1}$ denote the history of the algorithm till time step $t-1$, $y$ be a parameter $\in[0,1]$,  $p_{\phi
,t}=\mathbb{P}\brk*{\theta^t(a^\star)>y|\mathcal{H}_{t-1}}$ and $\tau_k$ denote the time step of the $k^{th}$ observation of the optimal arm, then we can bound the expectation of inverse of $p_{\phi
,\tau_k+1}$ as:
\begin{align*}
  \mathbb{E}\brk*{\frac{1}{p_{\phi
  ,\tau_k+1}}-1} &\leq \frac{3}{\Delta} &\text{ for } k<\frac{8}{\Delta}\\
  &\leq \Theta\Bigg(e^{\frac{-\Delta^2k}{2}}+\frac{1}{(k+1)\Delta^2}e^{-Dk}+\frac{1}{e^{\frac{\Delta^2k}{4}}-1}\Bigg) &\text{ for } k\geq \frac{8}{\Delta}
\end{align*}
where $\Delta=\mu(a^\star)-y$ and $D=y\ln\frac{y}{\mu(a^\star)}+(1-y)\ln\frac{(1-y)}{(1-\mu(a^\star))}$. 
\end{lemma} 
\begin{proof}[Proof of Lemma~\ref{lem:ts_conc_small_sample}]
To bound the left hand side, we use the fact that the samples of the optimal arm between two consecutive observations of the arm come from the same Beta distribution, since the distribution is not updated in between. We use the technique from Agrawal and Goyal \cite{Thompson-Sampling-Original} to bound the probability that the optimal arm has its sample far below its actual mean. This allows us to bound the number of times an arm $a\in\mathcal{V_{\phi}}$ can be selected while its sample is close to its mean because the sample of the optimal arm $a^\star$ is far enough below its mean $\mu(a^\star)$.

More formally, let $\mathcal{H}_{t-1}$ denote the history of the algorithm until the start of time step $t$. Using the fact that $\Delta(a)\leq 2^{-\phi
+1}$ for all $a\in\mathcal{V}_{\phi}$.
\begin{align*}
  &  \sum_{t=1}^T\sum_{a\in\mathcal{V}_{\phi  }}\mathbb{P}\brk*{a^t=a, N^t_{a}\geq L_{a}, \theta^t(a)\leq \mu(a)+\frac{\Delta(a)}{2}}\Delta(a)\\
  &\leq \sum_{t=1}^T\sum_{a\in\mathcal{V}_{\phi
  }}\mathbb{P}\brk*{a^t=a, N^t_{a}\geq L_{a}, \theta^t(a)\leq \mu(a)+\frac{\Delta(a)}{2}}2^{-\phi
  +1}\\
  &=\sum_{t=1}^T\mathbb{P}\brk*{a^t\in\mathcal{V}_{\phi
  }, N^t_{a^t}\geq L_{a^t}, \theta^t(a^t)\leq \mu(a^t)+\frac{\Delta(a^t)}{2}}2^{-\phi+1}\\
  &= \sum_{t=1}^T\mathbb{E}\brk*{\mathbb{P}\brk*{a^t\in\mathcal{V}_{\phi}, N^t_{a^t}\geq L_{a^t}, \theta^t(a^t)\leq \mu(a^t)+\frac{\Delta(a^t)}{2}|\mathcal{H}_{t-1}}}2^{-\phi+1},
  \end{align*}
  where the expectation is taken over the history $\mathcal{H}_{t-1}$.

Recall that we want to bound the probability of selecting a saturated arm in phase $\phi$ whose sample is bounded by $\theta^t(a)\leq \mu(a)+\frac{\Delta(a)}{2}$. Let $y_{\phi}=\max_{a\in\mathcal{V}_{\phi}}\prn*{\mu(a)+\frac{\Delta(a)}{2}}$ correspond to the upper bound on the sample of any such arm $a\in\mathcal{V}_{\phi}
$. Using Lemma~\ref{lem:ts_suboptimal_to_optimal}, we bound the above quantity  by:
  $$\leq \sum_{t=1}^T\mathbb{E}\brk*{
  \prn*{\frac{1}{p_{\phi,t}}-1}\mathbb{P}\brk*{a^t=a^\star|\mathcal{H}_{t-1}}}\cdot 2^{-\phi+1}
  $$ 
  where $p_{\phi,t}=\mathbb{P}\brk*{\theta^t(a^\star)>y_{\phi
  } |\mathcal{H}_{t-1}}$. 
Upper bounding the probability of selecting the optimal arm by the probability of observing it, we obtain: 
  \begin{align*}
  &\leq \sum_{t=1}^T\mathbb{E}\brk*{
  \prn*{\frac{1}{p_{\phi,t}}-1}\mathbb{P}\brk*{a^\star\in\mathcal{N}(a^t)\cup \crl{a^t}|\mathcal{H}_{t-1}}}\cdot 2^{-\phi+1}
  \end{align*}
By replacing probability of observing the optimal arm by expectation of the indicator function, the above is equal to:
\begin{align*}
  & \sum_{t=1}^T\mathbb{E}\brk*{
      \Bigg(\frac{1}{p_{\phi,t}}-1\Bigg)\mathbb{E}\brk*{\mathbbm{1}\{a^\star\in\mathcal{N}(a^t)\cup \crl{a^t}\}|\mathcal{H}_{t-1}}}\cdot 2^{-\phi+1} 
\end{align*}
Next expressing the expectation as a sum\footnote{The part of the history relevant to the algorithm is the outcome of the Bernoulli trials in each step, so there are only a finite set of possible histories.} over all possible histories $\mathcal{H}$, we obtain the following:
\begin{align*}
  &= \sum_{\mathcal{H}}\mathbb{P}\brk*{\mathcal{H}}\brk*{\sum_{t=1}^T 
      \Bigg(\frac{1}{p_{\phi,t}}-1\Bigg)\mathbbm{1}\{a^\star\in\mathcal{N}(a^t)\cup \crl{a^t}|\mathcal{H}\}}\cdot 2^{-\phi+1} 
\end{align*}
where the value $p_{\phi,t}$ inside the summation depends on the first $t-1$ steps of history $\mathcal{H}$.

Let $\tau_k$ be the time step for the $k^{th}$ observation of the optimal arm, a random variable depending on the history $\mathcal{H}$. Note that between two observations of the optimal arm, the distribution of the optimal arm does not change. Since $p_{\phi,t}$ does not depend on the random draws of any other arm, it therefore does not change between two observations of the optimal arm. Using this, the above quantity is equal to
\begin{align*}
  & \sum_{\mathcal{H}}\mathbb{P} \brk*{\mathcal{H}}
  \brk*{\sum_{k=0}^T\prn*{\frac{1}{p_{\phi
  ,\tau_k+1}}-1}\brk*{\sum_{t=\tau_k+1}^{\tau_{k+1}}\mathbbm{1}\{a^\star\in\mathcal{N}(a^t)\cup \crl{a^t}|\mathcal{H}\}}}\cdot 2^{-\phi+1}
\end{align*}
where the values $p_{\phi
,t}$ and $\tau_k$ inside the summation depend on the history $\mathcal{H}$ as before.

Further, for any history between $\tau_k+1$ and $\tau_{k+1}$ we have exactly one observation of optimal arm $a^\star$, i.e. $\abs{t\in[\tau_{k}+1,\tau_{k+1}):a^\star\mathcal{N}(a^t)\cup \crl{a^t}}=1$, by definition of $\tau_{k+1}$. As a result, the above sum can be expressed as:
  \begin{align*}
  &= \sum_{\mathcal{H}}\mathbb{P} \brk*{\mathcal{H}}\brk*{\sum_{k=0}^T\prn*{\frac{1}{p_{\phi,\tau_k+1}}-1}}\cdot 2^{-\phi+1}
  = \sum_{k=0}^T\mathbb{E}\brk*{\frac{1}{p_{\phi,\tau_k+1}}-1}\cdot 2^{-\phi+1}
\end{align*}

Recall again that $y_{\phi}=\max_{a\in\mathcal{V}_{\phi}}\prn*{\mu(a)+\frac{\Delta(a)}{2}}$ corresponds to the upper bound on the sample $\theta^t(a)$ of arms $a\in\mathcal{V}_{\phi}$. We use $\Delta_{\phi}=\mu(a^\star)-y_{\phi}$ to denote 
a lower bound on the gaps of the arms  $a\in\mathcal{V}_{\phi}$. We also denote by
$D_{\phi}=y_{\phi}\ln\frac{y_{\phi}}{\mu(a^\star)}+(1-y_{\phi})\ln\frac{(1-y_{\phi})}{(1-\mu(a^\star))}$ the KL-divergence between Bernoulli distributions with success probability $y_{\phi}$ and $\mu(a^\star)$.

Using Lemma~\ref{lem:expected_observations}, we can bound the above quantity by:
\begin{align*}
\sum_{k=0}^T\mathbb{E}\brk*{\frac{1}{p_{\phi,\tau_k+1}}-1}\cdot 2^{-\phi+1}\leq
\prn*{\frac{24}{\Delta_{\phi}^2}+\sum_{k\geq \frac{8}{\Delta_{\phi}}}^{T-1}\Theta\prn*{e^{\frac{-\Delta_{\phi}^2k}{2}}+\frac{1}{(k+1)\Delta_{\phi}^2}e^{-D_{\phi}k}+\frac{1}{e^{\frac{\Delta_{\phi}^2k}{4}}-1}} }\cdot 2^{-\phi}.
\end{align*}

Since $D_{\phi}$ corresponds to a KL-divergence, we can use the property that $D_{\phi}\geq0$, making $e^{-D_{\phi}k}\leq 1$. Combining this fact with the observation that ${e^{x}-1}\geq x$ for $x\geq 0$ and $\sum_{i=1}^T\frac{1}{(i+1)}\leq \log T$, we obtain:
  \begin{align*} 
   &\leq \prn*{\frac{24}{\Delta_{\phi
   }^2}+\sum_{k\geq \frac{8}{\Delta_{\phi
   }}}^{T-1}\Theta\prn*{e^{\frac{-\Delta_{\phi}^2k}{2}}+\frac{1}{(k+1)\Delta_{\phi}^2}+\frac{4}{k\Delta_{\phi}^2}} }\cdot 2^{-\phi}\\
   &=\Theta\prn*{\frac{1}{\Delta_{\phi}^2}+\frac{\log T}{\Delta_{\phi}^2}}\cdot 2^{-\phi}=O\prn*{\frac{\log T}{\Delta_{\phi}^2}}\cdot 2^{-\phi }
   \end{align*}
   Finally, using the fact that $\Delta_{\phi}=\mu(a^\star)-y_{\phi}\leq\frac{2^{-\phi+1}}{2}$, the above is upper bounded by $=\bigO\prn*{\frac{\ln T}{2^{-\phi}}}$ which completes the proof.
\end{proof}
\end{document}